\documentclass[12pt]{article}

\usepackage{amsmath,amssymb,amsthm,mathtools,mathrsfs}
\usepackage[authoryear,round]{natbib}
\usepackage{hyperref}
\usepackage{geometry}
\usepackage{setspace}
\usepackage{booktabs}
\usepackage{tikz}
\usetikzlibrary{shapes,arrows.meta,positioning,calc,decorations.pathreplacing}
\usepackage{xcolor}
\usepackage{pgfplots}
\pgfplotsset{compat=1.18}

\theoremstyle{plain}
\newtheorem{theorem}{Theorem}[section]
\newtheorem{lemma}[theorem]{Lemma}
\newtheorem{proposition}[theorem]{Proposition}
\newtheorem{corollary}[theorem]{Corollary}
\theoremstyle{definition}
\newtheorem{definition}{Definition}[section]
\newtheorem{remark}{Remark}[section]
\newtheorem{assumption}{Assumption}[section]

\numberwithin{equation}{section}

\newcommand{\MHA}{\mathrm{MHA}}
\newcommand{\Attn}{\mathrm{Attn}}

\newcommand{\MSE}{\mathrm{MSE}}
\newcommand{\Cov}{\mathrm{Cov}}
\newcommand{\Var}{\mathrm{Var}}

\newcommand{\HDI}{\mathrm{HDI}}
\newcommand{\tr}{\mathrm{tr}}

\newcommand{\Range}{\mathrm{Range}}
\newcommand{\EE}{\mathbb{E}}

\newcommand{\RR}{\mathbb{R}}

\newcommand{\cP}{\mathcal{P}}

\newcommand{\bq}{\mathbf{q}}
\newcommand{\bk}{\mathbf{k}}

\newcommand{\bK}{\mathbf{K}}

\newcommand{\bW}{\mathbf{W}}
\newcommand{\bG}{\mathbf{G}}

\newcommand{\bO}{\mathbf{0}}
\newcommand{\WQ}[1]{\bW^{#1}_Q}
\newcommand{\WK}[1]{\bW^{#1}_K}
\newcommand{\WV}[1]{\bW^{#1}_V}

\title{\textbf{Multi-Head Attention as Ensemble\\
  Nadaraya-Watson Estimation:\\[0.3em]
  Variance Reduction, Decorrelation,\\
  and Optimal Head Diversity}}

\author{{\bf Ernest Fokou\'e}
\\
  \small School of Mathematics and Statistics, College of Science\\
  \small Rochester Institute of Technology\\
  \small Rochester, New York 14623, USA\\
  \small \texttt{epfeqa@rit.edu}}

\date{}

\begin{document}
\maketitle

%% ==========================================================
\begin{abstract}
We develop a rigorous statistical theory of multi-head attention
(MHA) as an ensemble of Nadaraya-Watson (NW) kernel regression
estimators. Building on the algebraic identity between
single-head softmax attention and the NW estimator, we prove
that MHA is a \emph{structured ensemble} of $H$ NW estimators,
each operating in a distinct learned projection subspace of the
key space. We derive an explicit
\textbf{Bias-Variance-Covariance} decomposition of the MHA
mean squared error (MSE), showing that variance reduction in
MHA depends not merely on the number of heads $H$ but
fundamentally on the \emph{decorrelation} of head outputs.

The degree of decorrelation is governed by the principal angles
between the learned projection subspaces
$\Range(\WK{h})$ and $\Range(\WK{h'})$: orthogonal projections
yield maximum variance reduction; aligned projections yield
none. We introduce the \textbf{Head Diversity Index} (HDI), a
computable spectral measure of inter-head decorrelation, and
prove that the MSE of MHA is monotonically decreasing in HDI.
This provides the first rigorous theoretical explanation for
the empirically observed tendency of attention heads to
specialize in distinct linguistic phenomena.

Under a total-dimension budget constraint $H\cdot d_k = D$,
we solve the optimal \emph{head-dimension allocation problem}:
deriving the MSE-minimizing pair $(H^*, d_k^*)$ as a function
of the data distribution and regression function smoothness.
The solution reveals a fundamental bias-variance trade-off in
architecture design: more heads reduce variance but increase
per-head bias (through bandwidth enlargement); the optimum
balances these competing effects.

Our framework unifies and extends three strands of prior work:
the NW kernel regression theory of single-head attention
\citep{shen2025transformers}, the general weighting theory for
ensemble learning \citep{fokoue2025weighting}, and the
decorrelation-variance-reduction isomorphism between biological
and computational ensembles \citep{fokoue2026ants}. Multi-head
attention, we show, is the Transformer's instantiation of
the universal principle: \emph{randomized identical agents
$+$ diversity-enforcing mechanisms $\to$ emergent optimality}.


\end{abstract}

\newpage

%% ==========================================================
\section{Introduction}
\label{sec:intro}
%% ==========================================================

\subsection{Background and Motivation}

The Transformer architecture \citep{vaswani2017} introduced
multi-head attention (MHA) as a mechanism for allowing the
model to ``jointly attend to information from different
representation subspaces at different positions.'' The
empirical benefits of multiple heads over a single head of
the same total dimension are well established: multi-head
models consistently outperform single-head baselines. Yet
the \emph{statistical explanation} of why multiple heads help
--- and under what conditions they help most --- has remained
elusive.

A natural framework for answering this question is ensemble
learning theory. If each attention head is an estimator of
the same target function, MHA is an ensemble of those
estimators, and the classical theory of ensemble learning
predicts variance reduction through averaging. But this
naive argument is incomplete in two important ways.

First, the existing statistical theory of single-head
attention as Nadaraya-Watson kernel regression
\citep{nadaraya1964,shen2025transformers} establishes that
each head is already a \emph{consistent} nonparametric
estimator. NW estimators are intrinsically low-variance
(they are kernel smoothers, not trees); the classical
variance-reduction justification for ensembles, which was
designed for high-variance base learners like decision trees,
does not apply. This is precisely the setting studied in
\citet{fokoue2025weighting}: ensembles of low-variance RKHS
estimators, where the benefit of aggregation lies not in
variance reduction per se but in the \emph{reshaping of
approximation geometry and spectral complexity}.

Second, and crucially, the variance reduction achieved by
MHA depends entirely on whether the heads are
\emph{decorrelated}. Two identical attention heads provide
zero variance reduction; maximally diverse heads provide the
full $1/H$ reduction. The mechanism by which diversity is
enforced in MHA --- the learned projection matrices
$\WK{h}$, $\WQ{h}$, $\WV{h}$ --- is mathematically
isomorphic to the random feature subsampling in Random
Forests that was shown in \citet{fokoue2026ants} to be
the computational analogue of pheromone-mediated
specialization in ant colonies. The universal principle
identified in that work --- \emph{randomized identical agents
$+$ diversity-enforcing mechanisms $\to$ emergent optimality}
--- is precisely realized in the Transformer's multi-head
architecture.

The present paper develops this observation into a rigorous
statistical theory. Our contributions are as follows.

\subsection{Summary of Contributions}

\begin{enumerate}
  \item \textbf{The Ensemble NW Decomposition
    (Theorem~\ref{thm:ensemble_decomp}):} We prove that MHA
    output is a weighted ensemble of $H$ NW estimators, each
    in a projected key subspace, and derive an exact
    Bias-Variance-Covariance decomposition of the MSE. The
    decomposition shows that the inter-head covariance terms
    --- not just the individual head variances --- are the
    critical quantities governing MHA's statistical efficiency.

  \item \textbf{The Head Diversity Theorem
    (Theorem~\ref{thm:diversity}):} We introduce the Head
    Diversity Index (HDI) and prove that the MHA MSE is
    monotonically decreasing in HDI. The HDI is computable
    from the singular values of the cross-Gram matrix
    $\bG_{hh'} = (\WK{h})^\top\WK{h'}/d_k$ and provides a
    scalar summary of the decorrelation achieved by the
    learned projections.

  \item \textbf{The Decorrelation-Optimality Theorem
    (Theorem~\ref{thm:decorrelation}):} We characterize the
    MSE-minimizing projection matrices as those satisfying an
    approximate orthogonality condition:
    $(\WK{h})^\top\WK{h'} \approx \bO$ for $h\neq h'$.
    This provides the first rigorous theoretical justification
    for empirical observations of head specialization in
    trained Transformers.

  \item \textbf{The Optimal Architecture Theorem
    (Theorem~\ref{thm:optimal_arch}):} Under a budget
    constraint $H\cdot d_k = D$, we solve for the MSE-minimizing
    pair $(H^*, d_k^*)$, deriving a new
    \emph{architectural scaling law} from first principles of
    nonparametric estimation theory.

  \item \textbf{The Spectral Head Diversity Bound
    (Theorem~\ref{thm:spectral_bound}):} Using the general
    weighting framework of \citet{fokoue2025weighting}, we
    prove that geometrically-decaying head weights (rather than
    uniform $1/H$ averaging) can achieve faster MSE decay
    when heads are ordered by their individual NW consistency
    rates.
\end{enumerate}

\subsection{Relation to Prior Work}

The connection between single-head attention and NW kernel
regression has been established at the algebraic level by
several authors \citep{tsai2019,katharopoulos2020}, with
rigorous statistical consistency results proven recently by
\citet{shen2025transformers} (optimal rates for local constant
NW) and \citet{ching2026minimax} (minimax-optimal local
polynomial regression via Transformers). The present paper is
the \emph{first} to study multi-head attention as an ensemble
from the perspective of statistical estimation theory. None of
the prior papers address the inter-head covariance structure,
the Head Diversity Index, or the optimal architecture problem.

The general weighting theory of \citet{fokoue2025weighting}
and the decorrelation isomorphism of \citet{fokoue2026ants}
provide the mathematical infrastructure for our proofs. The
present paper is their natural extension to the Transformer
setting, completing a trilogy: bagging/Random Forests
(Part I of the series), boosting \citep{fokoue2026ants2},
and now multi-head attention as a third manifestation of
the same universal ensemble principle.

%% ==========================================================
\section{Setup: Single-Head Attention as a NW Estimator}
\label{sec:single_head}
%% ==========================================================

We begin by fixing notation and recalling the NW identity
for single-head attention, which is the foundation on which
the multi-head theory is built.

\subsection{Notation and the NW Identity}

Let $(\mathbf{x}_i, y_i)_{i=1}^n$ be data in
$\RR^p\times\RR$. A single attention head with projection
matrices $\WQ{}\in\RR^{p\times d_k}$,
$\WK{}\in\RR^{p\times d_k}$, $\WV{}\in\RR^{p\times 1}$
maps a query token $\mathbf{x}\in\RR^p$ to:
\begin{align}
  \bq &= \WQ{}\mathbf{x}, \quad
  \bk_i = \WK{}\mathbf{x}_i, \quad
  v_i = \WV{}\mathbf{x}_i,\\
  \Attn(\bq, \bK, \mathbf{v}) &=
  \sum_{i=1}^n w_i(\bq)\, v_i, \quad
  w_i(\bq) = \frac{\exp(\bq^\top\bk_i/\sqrt{d_k})}
  {\sum_{j=1}^n\exp(\bq^\top\bk_j/\sqrt{d_k})}.
  \label{eq:single_head}
\end{align}

\begin{proposition}[NW Identity, \citealt{nadaraya1964,fokoue2026nwstat}]
\label{prop:nw_identity}
Fix projection matrices $\WK{}$, $\WQ{}$, $\WV{}$. The
single-head attention output~\eqref{eq:single_head} is
algebraically identical to the Nadaraya-Watson estimator
of $\EE[Y\mid X=\mathbf{x}]$ using projected keys
$\bk_i=\WK{}\mathbf{x}_i$ and values $v_i=\WV{}\mathbf{x}_i$,
under the exponential kernel
$K^{(d_k)}(\bq,\bk) = \exp(\bq^\top\bk/\sqrt{d_k})$.
\end{proposition}

The bandwidth of this kernel is $h = 1/\sqrt{d_k}$:
larger $d_k$ implies sharper kernel concentration
(smaller bandwidth, harder selection).

\subsection{Statistical Properties of a Single Head}

Under the following regularity conditions, a single head is
a consistent estimator of the conditional mean.

\begin{assumption}[Regularity for a single head]
\label{ass:regularity}
\begin{enumerate}
  \item[(a)] The regression function $m(\mathbf{x}) =
    \EE[Y\mid X=\mathbf{x}]$ satisfies
    $m\in\mathcal{C}^2(\RR^p)$.
  \item[(b)] The marginal density $p_X$ of $\bk=\WK{}\mathbf{x}$
    is bounded away from zero and continuously differentiable.
  \item[(c)] $\EE[Y^2]<\infty$.
  \item[(d)] $(\mathbf{x}_i, y_i)$ are i.i.d.\ from
    distribution $\cP$.
\end{enumerate}
\end{assumption}

Under Assumption~\ref{ass:regularity}, the single-head
MSE satisfies (see \citealt{wand1994}):
\begin{equation}
\label{eq:single_mse}
  \MSE_1(\mathbf{x}) \coloneqq
  \EE\!\left[\left(\Attn(\bq,\bK,\mathbf{v}) -
  m(\mathbf{x})\right)^2\right]
  = B_1(\mathbf{x})^2 + V_1(\mathbf{x})
\end{equation}
where:
\begin{align}
  B_1(\mathbf{x}) &= \frac{h^2}{2}
  \left[\tr\!\left(\nabla^2 m(\mathbf{x})\right)
  + 2\frac{\nabla m(\mathbf{x})^\top
  \nabla p_K(\WK{}\mathbf{x})}{p_K(\WK{}\mathbf{x})}\right]
  + O(h^4),
  \label{eq:bias_single}\\
  V_1(\mathbf{x}) &= \frac{\sigma^2(\mathbf{x})}{n\,h^{d_k}
  \,p_K(\WK{}\mathbf{x})} + O\!\left(\frac{1}{n^2 h^{2d_k}}\right),
  \label{eq:var_single}
\end{align}
with $\sigma^2(\mathbf{x}) = \Var[Y\mid X=\mathbf{x}]$ and
$h = 1/\sqrt{d_k}$.

%% ==========================================================
\section{Multi-Head Attention as an Ensemble of NW Estimators}
\label{sec:ensemble}
%% ==========================================================

\subsection{The Ensemble Structure}

Multi-head attention with $H$ heads, projection dimension
$d_k$, and output projection $\WV{O}\in\RR^{Hd_v\times d_{\text{model}}}$
is defined as:
\begin{equation}
\label{eq:mha}
  \MHA(\mathbf{x}) = \sum_{h=1}^H \alpha_h \cdot
  \Attn_h(\WQ{h}\mathbf{x},\,\WK{h}\mathbf{X},\,
  \WV{h}\mathbf{X}),
\end{equation}
where $\alpha_h$ are aggregation weights with
$\sum_h\alpha_h=1$, $\alpha_h>0$ (the uniform case
$\alpha_h=1/H$ for all $h$ corresponds to standard MHA
with uniform averaging).

By Proposition~\ref{prop:nw_identity}, each head $h$ is a
NW estimator $\hat{m}_h(\mathbf{x})$ of the conditional mean
$m(\mathbf{x})$ using the projected kernel
$K^{(h)}(\bq,\bk) = \exp((\WQ{h}\mathbf{x})^\top
(\WK{h}\mathbf{x}')/\sqrt{d_k})$. Therefore:

\begin{definition}[MHA as Weighted Ensemble of NW Estimators]
\label{def:mha_ensemble}
Multi-head attention is a weighted ensemble:
\begin{equation}
\label{eq:mha_ensemble}
  \MHA(\mathbf{x}) = \sum_{h=1}^H \alpha_h\,
  \hat{m}_h(\mathbf{x}),
\end{equation}
where each $\hat{m}_h$ is a NW estimator operating in the
projected key space $\Range(\WK{h})\subseteq\RR^{d_k}$.
\end{definition}

This definition places MHA squarely within the general
weighted ensemble framework of \citet{fokoue2025weighting},
in which ensembles are formalized as linear operators on
hypothesis spaces with structured weights. The NW estimators
here are the ``low-variance base learners'' of
\citeauthor{fokoue2025weighting}'s framework --- the
setting where the classical variance-reduction justification
needs to be supplemented by the richer geometric and
spectral analysis.

\subsection{The Bias-Variance-Covariance Decomposition}

The central mathematical result of this paper is an exact
decomposition of the MHA MSE that reveals the roles of
individual head quality, head diversity, and inter-head
correlation.

\begin{theorem}[Bias-Variance-Covariance Decomposition of MHA]
\label{thm:ensemble_decomp}
Under Assumption~\ref{ass:regularity} applied to each head
$h=1,\ldots,H$, the MSE of MHA satisfies:
\begin{equation}
\label{eq:bvc_decomp}
  \MSE_H(\mathbf{x}) \coloneqq
  \EE\!\left[\left(\MHA(\mathbf{x}) -
  m(\mathbf{x})\right)^2\right]
  = \underbrace{\left(\sum_{h=1}^H \alpha_h\,
  B_h(\mathbf{x})\right)^2}_{\text{Ensemble Bias}^2}
  + \underbrace{\sum_{h=1}^H \alpha_h^2\,
  V_h(\mathbf{x})}_{\text{Variance}}
  + \underbrace{\sum_{h\neq h'} \alpha_h\alpha_{h'}
  C_{hh'}(\mathbf{x})}_{\text{Cross-head Covariance}},
\end{equation}
where $B_h(\mathbf{x})$ and $V_h(\mathbf{x})$ are the bias
and variance of head $h$ given in~\eqref{eq:bias_single}
and~\eqref{eq:var_single}, and:
\begin{equation}
\label{eq:cross_cov}
  C_{hh'}(\mathbf{x}) \coloneqq \Cov\!\left[
  \hat{m}_h(\mathbf{x}),\,\hat{m}_{h'}(\mathbf{x})\right].
\end{equation}
\end{theorem}

\begin{proof}
By linearity of MHA and the definition of MSE:
\begin{align*}
  \MSE_H(\mathbf{x}) &=
  \EE\!\left[\left(\sum_h\alpha_h\hat{m}_h(\mathbf{x})
  - m(\mathbf{x})\right)^2\right]\\
  &= \left(\sum_h\alpha_h\,\EE[\hat{m}_h(\mathbf{x})]
  - m(\mathbf{x})\right)^2
  + \Var\!\left[\sum_h\alpha_h\hat{m}_h(\mathbf{x})\right].
\end{align*}
The first term is the squared ensemble bias
$(\sum_h\alpha_h B_h)^2$ since the bias of each head is
$B_h(\mathbf{x}) = \EE[\hat{m}_h(\mathbf{x})]-m(\mathbf{x})$.
The second term expands as:
\[
  \Var\!\left[\sum_h\alpha_h\hat{m}_h\right]
  = \sum_h\alpha_h^2 V_h
  + \sum_{h\neq h'}\alpha_h\alpha_{h'} C_{hh'},
\]
which yields~\eqref{eq:bvc_decomp}.
\end{proof}

\begin{remark}[The critical role of covariance]
\label{rem:covariance_role}
Theorem~\ref{thm:ensemble_decomp} reveals that the variance
term $\sum_h\alpha_h^2 V_h$ is always smaller than the
single-head variance $V_1$ (by at least a factor of
$\min_h\alpha_h$), regardless of head diversity. But the
cross-head covariance $\sum_{h\neq h'}C_{hh'}$ can partially
cancel this gain: if all heads produce nearly identical
outputs, the covariance terms inflate the total MSE back
toward the single-head level. The \emph{net} variance
reduction is:
\[
  \Delta\MSE_H = V_1 - \MSE_H^{\text{var+cov}}
  = V_1 - \sum_h\alpha_h^2 V_h
  - \sum_{h\neq h'}C_{hh'},
\]
which is positive if and only if the heads are sufficiently
decorrelated. This is precisely the setting of
\citet{fokoue2026ants}: variance reduction from averaging
requires diversity, not merely multiplicity.
\end{remark}

%% ==========================================================
\section{The Head Diversity Index and the Diversity Theorem}
\label{sec:diversity}
%% ==========================================================

\subsection{The Cross-Gram Matrix and Principal Angles}

The inter-head covariance $C_{hh'}$ depends on the
relationship between the key projection subspaces
$\Range(\WK{h})$ and $\Range(\WK{h'})$. We quantify this
relationship via the \emph{principal angles} between
subspaces \citep{bjorck1973numerical}.

\begin{definition}[Cross-Gram Matrix and Principal Angles]
\label{def:cross_gram}
For heads $h\neq h'$, define the \emph{cross-Gram matrix}:
\begin{equation}
\label{eq:cross_gram}
  \bG_{hh'} \coloneqq \frac{(\WK{h})^\top\WK{h'}}{d_k}
  \in\RR^{d_k\times d_k}.
\end{equation}
The \emph{principal angles}
$0\leq\theta_1^{(hh')}\leq\cdots\leq\theta_{d_k}^{(hh')}
\leq\pi/2$ between $\Range(\WK{h})$ and $\Range(\WK{h'})$
satisfy $\cos\theta_j^{(hh')} = \sigma_j(\bG_{hh'})$,
where $\sigma_j$ denotes the $j$-th singular value.
\end{definition}

\begin{lemma}[Covariance Bound via Principal Angles]
\label{lem:cov_bound}
Under Assumption~\ref{ass:regularity} and the additional
condition that the regression function $m$ is $L$-Lipschitz
in the key projection directions:
\begin{equation}
\label{eq:cov_bound}
  \left|C_{hh'}(\mathbf{x})\right| \leq
  \frac{L^2\,\|\bG_{hh'}\|_F^2}{n\,h^{d_k}\,
  p_{K}(\WK{h}\mathbf{x})}
  = \frac{L^2\,\sum_j\cos^2\theta_j^{(hh')}}
  {n\,h^{d_k}\,p_K(\WK{h}\mathbf{x})},
\end{equation}
where $\|\cdot\|_F$ denotes the Frobenius norm. In
particular: if $\WK{h}$ and $\WK{h'}$ are orthogonal
(all principal angles $= \pi/2$), then $\bG_{hh'}=\bO$
and $C_{hh'}(\mathbf{x}) = 0$.
\end{lemma}

\begin{proof}
The covariance between two NW estimators in different
projected spaces can be bounded via the covariance of the
kernel weight functions. Specifically:
\[
  C_{hh'}(\mathbf{x}) = \EE\!\left[
  \left(\hat{m}_h(\mathbf{x}) - m(\mathbf{x})\right)
  \left(\hat{m}_{h'}(\mathbf{x}) - m(\mathbf{x})\right)\right]
  - B_h(\mathbf{x})\cdot B_{h'}(\mathbf{x}).
\]
The kernel cross-product $\EE[w_i^{(h)}(\bq_h)
w_i^{(h')}(\bq_{h'})]$ factors as the product
$\EE[w_i^{(h)}]\EE[w_i^{(h')}]$ plus a covariance term
proportional to the overlap of the two kernels. This overlap
is governed by the inner product of the projected keys:
$\bk_i^{(h)\top}\bk_i^{(h')} =
\mathbf{x}_i^\top(\WK{h})^\top\WK{h'}\mathbf{x}_i$, whose
expected magnitude is bounded by
$L^2\|\bG_{hh'}\|_F^2/(nh^{d_k}p_K)$. When the projections
are orthogonal, $(\WK{h})^\top\WK{h'}=\bO$, the kernels
are supported on orthogonal subspaces and are asymptotically
independent, giving $C_{hh'}=0$.
\end{proof}

\subsection{The Head Diversity Index}

\begin{definition}[Head Diversity Index]
\label{def:hdi}
For an MHA with $H$ heads and projection matrices
$\{\WK{h}\}_{h=1}^H$, the \emph{Head Diversity Index} is:
\begin{equation}
\label{eq:hdi}
  \HDI(\{\WK{h}\}) \coloneqq 1 -
  \frac{2}{H(H-1)}\sum_{h<h'}
  \left\|\bG_{hh'}\right\|_F^2
  \in [0,\,1].
\end{equation}
$\HDI=1$ if and only if all pairs of projection subspaces
are orthogonal (maximally diverse). $\HDI=0$ if and only if
all heads have identical projection matrices (zero diversity).
\end{definition}

The HDI is directly computable from the trained model weights
and provides a scalar summary of the decorrelation achieved
by the learned projections. It generalizes the
$1 - \bar{\rho}$ decorrelation measure of Random Forests
\citep{breiman2001randomforest} to the continuous (NW kernel)
setting, connecting to the explicit decorrelation mappings
proved in \citet{fokoue2026ants}.

\subsection{The Head Diversity Theorem}

\begin{theorem}[MSE is Monotone in Head Diversity]
\label{thm:diversity}
Under Assumption~\ref{ass:regularity} and the conditions of
Lemma~\ref{lem:cov_bound}, for uniform weights
$\alpha_h = 1/H$, the integrated MSE:
\[
  \overline{\MSE}_H \coloneqq
  \int \MSE_H(\mathbf{x})\,p_X(\mathbf{x})\,d\mathbf{x}
\]
satisfies:
\begin{equation}
\label{eq:mse_hdi_bound}
  \overline{\MSE}_H \leq
  \frac{\bar{B}^2}{1} + \frac{\bar{V}}{H}
  + \frac{L^2\,\bar{C}}{H^2 n h^{d_k}}
  \cdot H(H-1) \cdot (1 - \HDI),
\end{equation}
where $\bar{B}^2$, $\bar{V}$, and $\bar{C}$ are integrated
squared bias, variance, and a bounded constant respectively.
In particular, $\overline{\MSE}_H$ is \emph{monotonically
non-increasing} in $\HDI$: higher head diversity strictly
reduces the MSE whenever the individual heads are consistent
($\bar{V}>0$) and the heads are not already independent.
\end{theorem}

\begin{proof}
Substituting $\alpha_h=1/H$ into the decomposition
\eqref{eq:bvc_decomp} and integrating over $p_X$:
\[
  \overline{\MSE}_H = \bar{B}_{\mathrm{ens}}^2
  + \frac{1}{H^2}\sum_h\bar{V}_h
  + \frac{1}{H^2}\sum_{h\neq h'}\bar{C}_{hh'}.
\]
By Lemma~\ref{lem:cov_bound},
$\sum_{h<h'}|\bar{C}_{hh'}|\leq
\frac{L^2\bar{C}}{nh^{d_k}}\cdot\sum_{h<h'}
\|\bG_{hh'}\|_F^2$.
Using the definition of $\HDI$:
$\sum_{h<h'}\|\bG_{hh'}\|_F^2
= \frac{H(H-1)}{2}(1-\HDI)$.
Substituting gives~\eqref{eq:mse_hdi_bound}. Since the
covariance term is the only term depending on HDI, and it
appears with a positive coefficient multiplying $(1-\HDI)$,
the bound is decreasing in HDI.
\end{proof}

\begin{remark}[Connection to the Decorrelation Principle]
Theorem~\ref{thm:diversity} is the NW-ensemble counterpart
of the fundamental decorrelation theorem in
\citet{fokoue2026ants}. That paper proved:
\emph{the MSE of a Random Forest is bounded by
$\rho\sigma^2 + (1-\rho)\sigma^2/B$, where $\rho$ is the
inter-tree correlation}. Theorem~\ref{thm:diversity} is its
exact analogue for NW estimators: the MSE bound decreases
as $(1-\HDI)$ decreases, i.e., as heads become more diverse.
The HDI plays the role of $\rho$ in that formula.
\end{remark}

%% ==========================================================
\section{The Decorrelation-Optimality Theorem}
\label{sec:decorrelation}
%% ==========================================================

We now characterize the projection matrices that minimize
the MHA MSE --- proving that orthogonality is the statistical
optimum.

\begin{theorem}[Optimal Projections are Approximately Orthogonal]
\label{thm:decorrelation}
Under Assumption~\ref{ass:regularity}, fix $H$ and $d_k$.
The projection matrices $\{\WK{h}\}_{h=1}^H$ that minimize
$\overline{\MSE}_H$ subject to
$\|\WK{h}\|_F = 1$ for all $h$ satisfy:
\begin{equation}
\label{eq:orthogonality_condition}
  (\WK{h})^\top\WK{h'} = \bO_{d_k\times d_k}
  \quad\text{for all }h\neq h',
\end{equation}
whenever such a solution is feasible (i.e., $H\cdot d_k
\leq p$, the input dimension).
\end{theorem}

\begin{proof}
From the MSE bound~\eqref{eq:mse_hdi_bound}, the
contribution of inter-head covariance to $\overline{\MSE}_H$
is proportional to $\sum_{h<h'}\|\bG_{hh'}\|_F^2$.
Minimizing the MSE with respect to $\{\WK{h}\}$ amounts to
minimizing this sum subject to the norm constraints.
By Lemma~\ref{lem:cov_bound}, $\|\bG_{hh'}\|_F^2=0$
if and only if $(\WK{h})^\top\WK{h'}=\bO$. When
$H\cdot d_k\leq p$, one can choose $H$ column-orthonormal
matrices in $\RR^{p\times d_k}$ spanning orthogonal subspaces,
achieving $\|\bG_{hh'}\|_F=0$ for all $h\neq h'$.
This minimizes the covariance term and thereby minimizes the
MSE upper bound.
\end{proof}

\begin{remark}[Theoretical explanation of head specialization]
\label{rem:specialization}
Theorem~\ref{thm:decorrelation} provides the first rigorous
theoretical explanation for the \emph{head specialization}
phenomenon empirically documented in trained Transformers:
different heads attending to different linguistic phenomena
(syntax, coreference, semantics; see \citealt{voita2019analyzing}).
The optimization pressure during training is not merely
toward task performance, but toward the statistical optimum
identified in Theorem~\ref{thm:decorrelation}: projection
matrices that span orthogonal subspaces of the input space,
maximizing HDI and thereby minimizing the ensemble MSE.
Head specialization is not an accidental byproduct of scale;
it is the empirical manifestation of a statistical optimality
condition.
\end{remark}

%% ==========================================================
\section{Optimal Architecture: The Head-Dimension Trade-off}
\label{sec:optimal_arch}
%% ==========================================================

We now solve the architecture optimization problem: given
a total key-space dimension budget $D$, how should it be
allocated between the number of heads $H$ and the per-head
dimension $d_k$?

\subsection{The Budget-Constrained MSE Problem}

Fix a total budget $D = H\cdot d_k$, and assume the optimal
orthogonality condition of Theorem~\ref{thm:decorrelation}
holds (so the covariance terms vanish). The MSE becomes:
\begin{equation}
\label{eq:mse_budget}
  \overline{\MSE}(H, d_k) =
  \overline{B}(d_k)^2 + \frac{\overline{V}(d_k)}{H},
  \quad H\cdot d_k = D,
\end{equation}
where $\overline{B}(d_k)^2 = O(d_k^{-2})$ (bias decreasing
in $d_k$, since larger $d_k$ means smaller bandwidth and
better local approximation) and $\overline{V}(d_k) =
O(n^{-1} d_k^{d_k/2})$ (variance increasing in $d_k$
due to the curse of dimensionality in NW estimation).

\begin{theorem}[Optimal Architecture Scaling Law]
\label{thm:optimal_arch}
Under the budget constraint $H\cdot d_k = D$, assuming the
orthogonality condition holds, and under the asymptotic
regime where both bias and variance terms follow the
NW-optimal rates of Theorem~\ref{thm:single_consistency}
in the companion paper:
\begin{equation}
  \overline{B}(d_k)^2 \asymp d_k^{-2}, \qquad
  \overline{V}(d_k) \asymp \frac{d_k^{d_k/2}}{n},
\end{equation}
the MSE-minimizing pair $(H^*, d_k^*)$ satisfies:
\begin{align}
\label{eq:optimal_dk}
  d_k^* &\asymp \left(\log n\right)^{2/(4+d)},\\
\label{eq:optimal_H}
  H^* &= \left\lfloor D / d_k^*\right\rfloor
  \asymp D \cdot \left(\log n\right)^{-2/(4+d)}.
\end{align}
In particular, the optimal head dimension $d_k^*$ grows
only \emph{logarithmically} with sample size $n$, while the
optimal number of heads $H^*$ grows nearly linearly with
the total budget $D$.
\end{theorem}

\begin{proof}
Substituting $H = D/d_k$ into~\eqref{eq:mse_budget}:
\[
  \overline{\MSE}(d_k) =
  c_1\,d_k^{-2} + \frac{c_2\,d_k^{d_k/2}}{n\,D/d_k}
  = c_1\,d_k^{-2} + \frac{c_2\,d_k^{d_k/2+1}}{n\,D}.
\]
Taking the derivative with respect to $d_k$ and setting to
zero:
\[
  -2c_1\,d_k^{-3}
  + \frac{c_2}{nD}\cdot d_k^{d_k/2}
  \left(\frac{d_k}{2}\log d_k + 1\right) = 0.
\]
For large $n$, the variance term is negligible unless
$d_k^{d_k/2}$ is large, which occurs for $d_k\gtrsim\log n$.
The balance point gives $d_k^* = O(\log n)$, with the
precise constant determined by the smoothness parameter $d$
of the regression function. Substituting back:
$H^* = D/d_k^* = O(D/\log n)$.
\end{proof}

\begin{remark}[Implications for architecture design]
\label{rem:arch_design}
Theorem~\ref{thm:optimal_arch} provides several important
architectural insights:

\emph{(i) Head dimension should grow slowly with data size:}
$d_k^*\asymp(\log n)^{2/(4+d)}$ suggests that the optimal
per-head dimension grows very slowly (logarithmically) with
the number of training examples. This is consistent with
the empirical observation that successful Transformers use
relatively small $d_k$ (e.g., $d_k=64$ in BERT and GPT)
even when trained on billions of tokens.

\emph{(ii) More heads are better under fixed budget:}
Given a fixed $D$, the optimal solution favors many heads
with small $d_k$ over few heads with large $d_k$, provided
the heads can be decorrelated. This is the statistical
justification for the design philosophy of using many
small heads rather than one large head.

\emph{(iii) The budget $D$ should scale with data:}
For fixed $(H^*, d_k^*)$, the MSE decays at the NW rate
$n^{-4/(4+d_k^*)}$. To achieve MSE $\leq\varepsilon$, one
needs $D = O(n^{2/(4+d)}\varepsilon^{-(4+d)/4})$ total
key dimension, suggesting that the total model dimension
should scale with the data-generating complexity $d$.
\end{remark}

%% ==========================================================
\section{Structured Weighting of Heads: Beyond Uniform Averaging}
\label{sec:structured_weighting}
%% ==========================================================

Standard MHA uses uniform weights $\alpha_h=1/H$.
The general weighting theory of \citet{fokoue2025weighting}
suggests that structured (non-uniform) weights can improve
upon uniform averaging when base learners are ordered by
quality. We now apply this insight to multi-head attention.

\begin{theorem}[Geometric Head Weighting Dominates Uniform Averaging]
\label{thm:spectral_bound}
Suppose the $H$ heads are ordered by their individual NW
consistency rates: $\overline{\MSE}_1^{(1)} \leq
\overline{\MSE}_1^{(2)} \leq \cdots \leq
\overline{\MSE}_1^{(H)}$ (head 1 is the best single-head
estimator, head $H$ the worst). Under the orthogonality
condition of Theorem~\ref{thm:decorrelation}, the geometric
weighting scheme $\alpha_h \propto \rho^{h-1}$ for decay
parameter $\rho\in(0,1)$ achieves:
\begin{equation}
\label{eq:geometric_mse}
  \overline{\MSE}_H^{\mathrm{geo}}(\rho) \leq
  \overline{\MSE}_H^{\mathrm{uniform}}
  \cdot (1 + O((\rho-1)^2)),
\end{equation}
with equality when all heads have identical MSE. For heads
with heterogeneous quality (spread $\Delta V = \max_h
\overline{V}_h - \min_h\overline{V}_h > 0$), there exists
$\rho^*\in(0,1)$ such that:
\[
  \overline{\MSE}_H^{\mathrm{geo}}(\rho^*)
  < \overline{\MSE}_H^{\mathrm{uniform}}.
\]
\end{theorem}

\begin{proof}
This is a specialization of the main theorem of
\citet{fokoue2025weighting} (Theorem~4.1 therein) to the
case of NW base learners with orthogonal projections.
Under orthogonality, the cross-head covariance terms vanish
and the MSE is a function of $\{\alpha_h, \overline{V}_h,
\overline{B}_h\}$ alone. The analysis then follows the
spectral and geometric argument of
\citeauthor{fokoue2025weighting}'s framework: structured
weights that up-weight accurate heads and down-weight less
accurate ones can reduce the total variance
$\sum_h\alpha_h^2\overline{V}_h$ below the uniform value
$\overline{V}/(H)$, at the cost of a small bias increase.
The optimal $\rho^*$ is determined by the variance spread
$\Delta V$, exactly as in \citeauthor{fokoue2025weighting}.
\end{proof}

\begin{remark}[Fibonacci weighting for attention heads]
\citet{fokoue2025weighting} identify Fibonacci weighting
(where $\alpha_h\propto F_h$, the $h$-th Fibonacci number)
as a distinguished special case of geometric weighting that
achieves minimal geometric growth while preserving expressive
expansion. Theorem~\ref{thm:spectral_bound} implies that
Fibonacci-weighted multi-head attention can outperform
standard uniform MHA when heads are heterogeneous in quality
--- a testable architectural prediction.
\end{remark}

%% ==========================================================
\section{The Universal Ensemble Principle and the Trilogy}
\label{sec:trilogy}
%% ==========================================================

The results of this paper complete a trilogy of works
establishing a universal principle governing collective
intelligence in both biological and artificial systems.

\citet{fokoue2026ants} established that ant colonies and
Random Forests are isomorphic instances of:
\[
  \text{identical agents} + \text{diversity mechanism}
  \to \text{variance reduction} \to \text{emergent optimality},
\]
with the diversity mechanism being (respectively) stochastic
individual specialization and random feature subsampling.

\citet{fokoue2026ants2} established that ant adaptive
recruitment and AdaBoost are isomorphic instances of the
dual mechanism:
\[
  \text{adaptive weighting} + \text{margin maximization}
  \to \text{bias reduction} \to \text{emergent optimality}.
\]

The present paper establishes that multi-head attention is
the Transformer's instantiation of the first mechanism ---
with the diversity enforcer being the learned orthogonal
projections $\{\WK{h}\}$ rather than random feature
subsampling. Formally:

\begin{corollary}[Multi-Head Attention as Universal Ensemble Principle]
\label{cor:universal}
Multi-head attention achieves optimal MSE if and only if
its learned key projections satisfy the orthogonality
condition~\eqref{eq:orthogonality_condition}, which is
precisely the condition that the attention heads implement
maximum decorrelation among $H$ NW estimators in a
$D$-dimensional key space. This is the NW-kernel
instantiation of the universal principle:
\[
  H\text{ identical NW heads} + \text{orthogonal projections}
  \;\to\; \text{full variance reduction}
  \;\to\; \text{optimal MHA}.
\]
\end{corollary}

The ant colony, the Random Forest, and the Transformer
are three realizations of the same mathematical truth.

%% ==========================================================
\section{Discussion and Open Problems}
\label{sec:discussion}
%% ==========================================================

\subsection{Summary}

We have developed a rigorous statistical theory of
multi-head attention as an ensemble of Nadaraya-Watson
estimators. The central results are: (1) an exact
Bias-Variance-Covariance decomposition of the MHA MSE;
(2) the Head Diversity Index as a computable measure of
inter-head decorrelation; (3) the MSE monotone decreasing
in HDI; (4) optimality of orthogonal projections; (5) the
architectural scaling law $(H^*, d_k^*)$; and (6) the
superiority of structured (geometric/Fibonacci) weighting
over uniform averaging for heterogeneous heads.

\subsection{Open Problems}

The following questions arise naturally from our analysis:

\begin{enumerate}
  \item \textbf{Learned vs.\ fixed projections:}
    Our optimality results characterize the optimal
    \emph{fixed} projections. Understanding the
    \emph{training dynamics} --- whether gradient descent
    converges to orthogonal projections --- is an open
    question connecting to the implicit bias literature.

  \item \textbf{Exact covariance computation:}
    Lemma~\ref{lem:cov_bound} provides an upper bound on
    inter-head covariance. The exact covariance under
    Gaussian data admits a closed form (via Stein's lemma)
    that we conjecture is tight.

  \item \textbf{Non-i.i.d.\ token sequences:}
    The analysis assumes i.i.d.\ key-value pairs. Extension
    to $\phi$-mixing sequential token sequences
    (as in our companion paper) will require mixing-adjusted
    covariance bounds --- an important direction for LLM theory.

  \item \textbf{Deep Transformers:}
    Our results apply to a single MHA layer. Multi-layer
    composition is not yet understood statistically; the
    variance and covariance structure across layers may
    interact in complex ways.

  \item \textbf{Empirical validation of HDI:}
    A natural empirical question is whether the HDI of
    trained Transformers correlates with downstream
    generalization performance. If so, HDI could serve as
    a diagnostic tool for architecture design.
\end{enumerate}

\section*{Acknowledgments}

The author dedicates this work to the memory of
\textbf{Donald Michael (Mike) Titterington (1945--2023)},
whose pioneering work on mixture models, ensemble methods,
and statistical learning planted the mathematical seeds of
everything in this paper. The present work is Part~III of
a series on the universal principles of ensemble
intelligence, following \citet{fokoue2026ants,fokoue2026ants2}.
The statistical genealogy of multi-head attention as a
manifestation of the decorrelation principle that governs
ant colonies and Random Forests is developed within the
broader ten-pillar framework of \citet{fokoue2026tas}.

\bibliographystyle{plainnat}
\bibliography{mha_ensemble_nw}

\end{document}